\documentclass[10pt,journal,compsoc]{IEEEtran}
\ifCLASSOPTIONcompsoc
\else
\fi
\ifCLASSINFOpdf
\else
\fi

\usepackage{amsmath,amssymb,amsfonts}
\usepackage{graphicx}
\usepackage{textcomp}

\usepackage{amsthm}
\usepackage{epsfig}
\usepackage{bm}
\usepackage[ruled,norelsize]{algorithm2e}

\usepackage{tabularx}
\usepackage[dvipsnames]{xcolor}
\usepackage{comment}

\newtheorem{theorem}{Theorem}
\newtheorem{corollary}{Corollary}

\makeatletter
\newcommand{\removelatexerror}{\let\@latex@error\@gobble}
\makeatother

\begin{document}
%
\title{Rank-R FNN: A Tensor-Based Learning Model for High-Order Data Classification}
%
%
%
%

\author{Konstantinos Makantasis, Alexandros Georgogiannis, Athanasios Voulodimos, Ioannis Georgoulas, Anastasios Doulamis, Nikolaos Doulamis
\IEEEcompsocitemizethanks{\IEEEcompsocthanksitem K. Makantasis is with the Institute of Digital Games, University of Malta, Msida, Malta \protect\\
E-mail: konstantinos.makantasis@um.edu.mt
\IEEEcompsocthanksitem A. Georgogiannis is with the School of Electrical and Computer Engineering, Technical University of Crete, Greece, and DeepLab, Athens, Greece, E-mail: alexgj@deeplab.ai
\IEEEcompsocthanksitem A. Voulodimos is with the Department of Informatics and Computer Engineering, University of West Attica, Greece, 
\protect\\ E-mail: avoulod@uniwa.gr
\IEEEcompsocthanksitem I. Georgoulas, A. Doulamis and N. Doulamis are with theSchool of Rural and Surveying Engineering, National Technical University of Athens, Greece, E-mail: \{adoulam, ndoulam\}@cs.ntua.gr}
}

\IEEEtitleabstractindextext{%
\begin{abstract}
An increasing number of emerging applications in data science and engineering are based on multidimensional and structurally rich data. The irregularities, however, of high-dimensional data often compromise the effectiveness of standard machine learning algorithms. We hereby propose the Rank-$R$ Feedforward Neural Network (FNN), a tensor-based nonlinear learning model that imposes Canonical/Polyadic decomposition on its parameters, thereby offering two core advantages compared to typical machine learning methods. First, it handles inputs as multilinear arrays, bypassing the need for vectorization, and can thus fully exploit the structural information along every data dimension. Moreover, the number of the model's trainable parameters is substantially reduced,  making it very efficient for small sample setting problems. We establish the universal approximation and learnability properties of Rank-$R$ FNN, and we validate its performance on real-world hyperspectral datasets. Experimental evaluations show that Rank-$R$ FNN is a computationally inexpensive alternative of ordinary FNN that achieves state-of-the-art performance on higher-order tensor data.
\end{abstract}

}

\maketitle

\IEEEdisplaynontitleabstractindextext

%

\section{Introduction}
\label{sec:introduction}
\PARstart{L}{arge} sets of high-order data have become ubiquitous across science and engineering disciplines, primarily due to recent advances in sensing technologies and increasingly affordable recording devices. Remote Sensing is not an exception, where large hyperspectral data --collections of high-order images-- are becoming available and used for a variety of applications, including urban and rural planning, change detection, mapping, geographic information systems, monitoring, housing value, and navigation \cite{doulamis20155d}, \cite{qin2021small}, \cite{lee2014applicability}.

High-order data is produced either when data itself is collected in a multi-linear format, or when low-order data is molded into high-order structures \cite{zadeh2017tensor,makantasis2015tunnel,nikitakis2019unified}. The information encoded in high-order data exhibits strong correlations across  different modes, i.e., the modes of matrices or tensors used  \cite{cichocki2017tensor}. Although such correlations favor data analysis techniques, the structural complexity of acquired information renders standard machine learning algorithms inadequate for its analysis \cite{zhou2016linked,camps2005kernel}. 

In particular, most machine learning algorithms assume that their input is in vector form. There exist, however, cases, such as image analysis, where vectorization of tensor input deteriorates the performance of standard data analysis algorithms since it destroys any inherent structural information that may be present in data.\cite{li2018tucker} (spatial and/or spectral coherency). Another drawback of  vectorization, without imposing additional structural constraints, is the production of large high-dimensional parameter spaces which compromise both computational efficiency and theoretical guarantees of vector-input machine learning methods \cite{zhou2013tensor}.

To overcome the problems related to vectorization, Convolutional Neural Networks (CNN) process multidimensional inputs without vectorizing them. Specifically, via a sequence of convolutions and nonlinear transformations, CNN map multidimensional inputs to vector representations used for classification purposes. Especially in hyperspectral image data classification, CNN prove to be very accurate and robust classifiers, as shown in \cite{makantasis2015deep}, one of the most influential works in remote sensing. However, the main drawback of CNN is the large number of trainable parameters and, consequently, the large number of training samples needed to achieve accurate classification performance.

Motivated by the limitations above, we propose an alternative learning paradigm in pattern recognition of high-order data. We introduce a tensor-based non-linear learning model, henceforth called Rank-R Feedforward Neural Network (FNN). The Rank-$R$ FNN imposes a Canonical/Polyadic (CP) decomposition of rank $R$ on its weights, which leads to a dramatic reduction of the number of parameters to be estimated during training. Consider for example a fully connected FNN with one hidden layer and $h$ hidden neurons that receives as input a 3-order tensor object $\bm X \in \mathbb R^{p_1 \times p_2 \times p_3}$. The cardinality of the weights set that connect the input to the hidden layer is $h \prod_{i=1}^{3} p_i$, while the cardinality of the corresponding set for the Rank-$R$ FNN is $h R \sum_{i=1}^{3} p_i$. Moreover, the Rank-$R$ FNN processes covariates in tensor format in an attempt to exploit, as much as possible, any structural richness presented in data and reveal any correlations residing across different tensor modes. To summarize, the main advantage of the proposed rank-$R$ FNN model is, first, the dramatic reduction of the number of model parameters and, second, the exploitation of the inputs' structural information. These two properties shield the proposed model against overfitting, making it ideal for accurate classification when a limited number of training examples are available.

Hyperspectral data classification is a typical small sample setting problem. Collecting large annotated hyperspectral corpora is a tedious and high-cost task since a group of human experts should visit the place depicted in a remotely sensed image and manually annotate the displayed materials\cite{protopapadakis2021stacked}. Therefore, this paper investigates the proposed Rank-$R$ FNN model's capacity to classify hyperspectral data accurately and compares it against typical machine learning schemes such as CNN.

\subsection{Related Work}
Several supervised and unsupervised learning methods have been proposed for analyzing data in tensor format, including High Order SVD, Tucker and CP decompositions, \cite{kolda2009tensor}, Multi-linear PCA \cite{lu2008mpca}, probabilistic decompositions \cite{chu2009probabilistic,rai2014scalable,xu2015bayesian}, and Common Mode Patterns \cite{makantasis2019common}. Such methods, known as subspace learning, project raw data into lower dimensional spaces and consider these projections as highly descriptive features of raw information. However, there is no consensus on what choice of features best summarizes a learning task \cite{li2018tucker}. In the supervised learning setup, subspace learning methods are often utilized as a preprocessing step \cite{makantasis2015deep}, but they come with one key limitation: they do not take into account the labels of the data and, as a consequence, they produce features with limited discrimination power for classification or regression tasks.

Tensor-based supervised learning methods for high-order data have also been proposed in \cite{tao2007general,tan2012logistic,zhou2013tensor,li2014multilinear,hoff2015multilinear}. These methods generate linear relations between the input and the desired output, and thus, they poorly handle complex input-output statistical relations that require nonlinear maps. Nonlinear tensor classification models, such as the Rank-1 FNN, were introduced recently in \cite{makantasis2018icassp}. Rank-1 FNN is a Fully Connected Feedforward Neural Network (FCFNN) whose weights satisfy a rank-1 CP decomposition~\cite{makantasis2018tensor}. Rank-1 FNN, however, comes with one drawback: the output of the first hidden layer can only represent features that lie within axis-aligned rectangles (for details see \cite{zhou2013tensor}). The current study overcome this drawback by not restricting the rank of the CP decomposition, which is imposed on the model's parameter, to be equal to one.

Kossaifi \textit{et al.} in \cite{kossaifi2017tensor}, propose a tensor-based neural network, which employs tensor contraction layers to propagate the information of tensor inputs through the layers of the network, and a tensor regression layer as the output layer. The model in \cite{kossaifi2017tensor} and the proposed Rank-$R$ FNN model may look similar, they are, however, essentially different. Specifically,the sequence of tensor contraction layers employed in \cite{kossaifi2017tensor} perform a sequence of nonlinear tensor projections, i.e., project a tensor object to another tensor subspace. Therefore,the model in \cite{kossaifi2017tensor} retains the tensor form of the information through all its layers. On the contrary, in the case of Rank-$R$ FNN, we project the tensor objects to a \textit{vector} space by imposing a CP decomposition of rank $R$ on the weights that connect the input to each one of the neurons of the first hidden layer. This way, we produce a compact yet highly informative representation of the inputs, and after the first hidden layer, we are able to propagate the information in a similar manner as in typical fully connected feed-forward neural networks.  

Besides the derivation of linear and nonlinear tensor-based learning models, the importance of tensor algebra tools is also emphasized via their exploitation towards the compression of very deep neural network architectures \cite{novikov2015tensorizing,lebedev2014speeding,garipov2016ultimate}, as well as towards the investigation of theoretical properties of deep learning machines \cite{khrulkov2017expressive,cohen2016expressive}. The studies dealing with the compression of deep learning architectures exploit tensor decompositions to reduce the number of the parameters of \textit{already trained} networks with minor accuracy drop. In contrast to these works, in this study we derive a learning model whose parameters are inherently compressed and this compression is retained during the training phase. In other words, the Rank-$R$ FNN is not derived by compressing another already trained model, but it is trained from scratch. On the other hand, the studies focusing on the theoretical properties of deep neural networks exploit tensor algebra tools to conclude about the expressive power of known architectures, such as convolutional and recurrent neural networks. The presented study is complementary to the studies mentioned above, since it focuses on the theoretical properties of the proposed Rank-$R$ FNN model. 

\subsection{Our Contribution}
The main contributions of this work are as follows. We extend Rank-1 FNN to Rank-$R$ FNN, a nonlinear classifier for tensor data which imposes a CP decomposition constraint of rank $R$ on its weights. Allowing weights to satisfy CP decompositions of rank higher than one increases the representation power of Rank-$R$ FNNs allowing them to model complex input-output statistical relations. In addition, this study significantly extends the work of \cite{makantasis2019hyperspectral} in two different directions. First, we investigate the theoretical properties of Rank-$R$ FNN models and prove, on the one hand, that they have universal approximation properties, and, on the other, that the class of functions they implement can be efficiently learned by the empirical risk minimization principle. Second, we experimentally investigate the robustness of Rank-$R$ FNN on noisy data, as well as its behavior with respect to different architecture design configurations. The observed performance of Rank-$R$ FNN on real-world high-order hyperspectral image datasets indicates that it  achieves state-of-the-art results on small sample setting problems, where the number of labeled examples is limited. Learning from a limited number of training examples is one of the most important properties of the proposed model, since in many real-world applications, such as in hyperspectral data classification, collecting large annotated corpora is a tedious and costly task. 

As we summarize in  Section \ref{sec:conclusions}, the advantages of Rank-$R$ model for classifying hyperspectral data are the following. (A) Our model requires times smaller number of trainable parameters which make it suitable for handling a small amount of training samples. (B) It presents robust  classification accuracy both for noise-free and noisy data inputs. (C) It converges very rapidly in contrast to CNN, which requires many epochs to reach a plateau, and (D) it is very robust against different execution runs in terms of converging to the best solution.  All these advantages have been revealed by applying the proposed Rank-$R$ FNN model for hyperspectral data classification over benchmarked datasets.

The remainder of this paper is structured as follows. Section 2  introduces the notation and states the problem formulation. Section 3 presents Rank-$R$ FNN, while Section 4 explores its theoretical properties. In Section 5, we experimentally evaluate the proposed model, and in the last section, Section 6, we conclude with a summary of findings.

\section{Nomenclature and Problem Formulation}
We hereby introduce the notation, definitions and tensor algebra operations to be used through out  this study. After that, we formulate the problem to be addressed.

\subsection{Tensor Algebra Notation}
The following definitions introduce basic operations pertaining to high-order tensor processing. We focus on the operations that are used through out this study. For an thorough introduction in tensor algebra refer to the excellent survey on  higher-order tensor decompositions in \cite{kolda2009tensor}. 
In what follows, tensors and vectors are denoted in bold uppercase and bold lowercase letters, respectively, while and scalars are denoted in lowercase letters.

\vspace{0.03in}
\noindent \textbf{Tensor vectorization}. The $\text{vec}(\bm B)$ operator stacks the entries of a $D$-order tensor $\bm B \in \mathbb R^{p_1 \times \cdots \times p_D}$ into a column vector. 
That is, entry $\bm B=[\cdots b_{i_1, \cdots, i_D}\cdots]$ maps to the $j^{th}$ entry of $\text{vec}(\bm B)$, in which $j=1+\sum_{d=1}^D(i_d-1)\prod_{d'=1}^{d-1}p_{d'}$.

\vspace{0.03in}

\noindent \textbf{Tensor inner product}. The inner product of two tensors $\bm A, \bm B \in \mathbb R^{p_1 \times \cdots \times p_D}$ with $D > 2$ is defined as
\begin{equation}
	\begin{split}
		\langle \bm A, \bm B \rangle &= \langle \text{vec}(\bm A), \text{vec}(\bm B) \rangle \\ &= \sum_{i_1=1}^{p_1} \sum_{i_2=1}^{p_2} \cdots \sum_{i_D=1}^{p_D}  a_{i_1 i_2\cdots i_D}  b_{i_1 i_2\cdots i_D},
	\end{split}
\end{equation}
where $i_1, \cdots, i_D$ are the indices of tensors' elements.

\vspace{0.03in}

\noindent \textbf{Tensor matricization}. The mode-\textit{d} matricization, $\bm B_{(d)}$, maps a tensor $\bm B$ into a $p_d \times \prod_{d' \neq d}p_{d'}$ matrix by arranging the mode-\textit{d} fibers to be the columns of the resulting matrix. That is, the $(i_1,\cdots,i_D)$ element of $\bm B$ maps to the $(i_d,j)$ element of $\bm B_{(d)}$, where $j=1+\sum_{d' \neq d}(i_{d'}-1)\prod_{d''<d',d'' \neq d}p_{d''}$.

\vspace{0.03in}
\noindent \textbf{Rank-\textit{R} CP decomposition}. A tensor $\bm B \in \mathbb R^{p_1 \times \cdots \times p_D}$ admits a rank-\textit{R} CP decomposition if  $\bm B = \sum_{r=1}^R \bm b_1^{(r)} \circ \cdots \circ \bm b_D^{(r)}$, where $\bm b_d^{(r)}\in \mathbb R^{p_d}$.
This decomposition is denoted as
$\bm B = [\![ \bm B_1, ... ,\bm B_D ]\!]$, where $\bm B_d = [\bm b_d^{(1)}, ... ,\bm b_d^{(R)}] \in \mathbb R^{p_d \times R}$. 
When a tensor $\bm B$ admits such a  decomposition, it holds true that
\begin{equation}
	\label{eq:1}
	\bm B_{(d)} = \bm B_d(\bm B_D \odot \cdots \odot \bm B_{d+1} \odot \bm B_{d-1} \odot \cdots \odot \bm B_1)^T,
\end{equation}

where $\odot$ is the Khatri-Rao product.

\vspace{0.03in}
\noindent \textbf{Vector tensorization}. The $\text{ten}(\bm b)$ operator transforms a vector $\bm b \in \mathbb R^{p_1 p_2 \cdots p_D}$ into a tensor $\bm B \in \mathbb R^{p_1 \times p_2 \times \cdots \times p_D}$, such that the element of $\bm B$ indexed by $i_1, i_2, \cdots i_D$ is the $i$-th element of $\bm b$, where $i=1 + \sum_{d=1}^D(i_d -1)\prod_{k=1}^{d-1}p_k$.

\subsection{Problem Formulation}
Let $\bm X \in \mathbb R^{I_1 \times \cdots \times I_D}$
a $D$-order random tensor
and $\bm X_i \in \mathbb R^{I_1 \times \cdots \times I_D}$
independent copies of $\bm X$ from the same probability distribution as $\bm X$.
Consider a family of learning models, each one parameterized by a set of parameters $\theta$.
The set $\theta$ constitutes a  complete description of a model from the family and contains all parameters defining  it.
We assume that there are $C$ available classes and we aim at classifying
$\bm X$ into one of these classes using the models at hand.

Let $\bm p(\bm X; \theta)$  denote a $C$-dimensional real vector whose
coordinates sum  to one and its $k$-th element,
$p^k(\bm X; \theta)$,
for $k=1,\dots, C,$ expresses the probability that $\bm X$ belongs to the $k$-th class.
Function $\bm p(\cdot;\theta)$, when regarded as a function of $\bm X$,
is an approximation of the conditional probability that $\bm X$ belongs to the $k$-th class.
The optimal value of $\theta$ is estimated by empirical risk minimization over the training dataset
\begin{equation}
	\label{eq:dataset}
	\mathcal D = \{(\bm X_i, \bm t_i)\}_{i=1}^N.
\end{equation}
Here  $\bm t_i = [t_{i,1}, \cdots, t_{i,C}]^T \in \{0,1\}^C$ is a unitary vector with $\sum_k t_{i,k}=1$ and indicates the class to which  $\bm X_i$ actually belongs\footnote{When $(\bm X_i, \bm t_i)$ is considered as a random tuple,
	the joint probability measure $P\{(\bm X_i , \bm t_i )\}$ admits the factorization
	$P\{X_i\}P\{t_i|X_i\}$.}.
In the following, whenever we omit subscript $i$ from a tensor,
we just refer to an input sample.
Having an estimation for parameter $\theta$,
our final decision for $\bm X$ is
\begin{equation}
	k_i^* = \arg \max_{k=1,\dots, C} \bm p^k(\bm X; \theta).
\end{equation}
The ultimate goal in classification is to calculate the set of model parameters $\theta$ where the minimum of
\begin{equation}
	\label{eq:objective}
	\sum_i \mathbb I(k_i^* \neq \arg \max_k \bm t_i) \rightarrow \text{min}
\end{equation}
is attained; $\mathbb I(\cdot)$ in (\ref{eq:objective}) stands for the indicator function and takes values in $\{0,1\}$.
This is the primary goal of any classification scheme. However, in our case $\bm X$ is a tensor, and so are many elements of the parameter set $\theta$.
Attaining the minimum in (\ref{eq:objective}) turns out to be an NP-hard problem (see also  Ch. 12 in \cite{shalev2014understanding}).
For that reason, we use instead the negative log-likelihood as a surrogate loss function to approximate the objective in (\ref{eq:objective}).
Next, we describe a classifier for tensors, whose input and model parameters retain their tensor form.

\section{Tensor-Based Rank-$R$ Nonlinear Classifier}
In this Section, we briefly introduce Rank-1 FNN and describe in detail its extension, the Rank-$R$ FNN classifier.

\subsection{Rank-1 FNN Modeling}
Rank-1 FNN is a two layer neural network  which models the  weights  connecting
input layer to  hidden layer as:
\begin{equation}
	\label{eq:rank_1_decomposition}
	\bm w^{(q)} = \bm w_D^{(q)} \circ \cdots \circ \bm w_1^{(q)}
	\in \mathbb R^{I_1 \times \cdots \times I_D},
\end{equation}
with $\bm w_d^{(q)} \in \mathbb R^{I_d}$,
$d=1,\cdots,D,$ denoting  weights which connect the input to the $q$-th neuron of the hidden layer.
The activation function $g(\cdot)$ of the $q$-th hidden neuron receives
$\langle \bm w^{(q)}, \bm X \rangle  $,
as input and outputs
\begin{equation}
	\begin{split}
		u_q &= g(\langle\bm w^{(q)}, \bm X \rangle)\\ &= g(\langle (\bm w_D^{(q)} \circ \cdots \circ \bm w_1^{(q)}), \bm X \rangle) \in \mathbb{R};
	\end{split}
\end{equation}
thus, the output of Rank-1 FNN is
\begin{equation}
	\label{eq:rank_1_output}
	p^{k} =\sigma ( \langle \bm v^{(k)},\bm u \rangle).
\end{equation} 
Here $\sigma(\cdot)$ denotes the softmax activation function, $\bm u = [u_1, u_2, \cdots, u_Q]^T$, $\bm v^{(k)} $ collects the weights between the hidden and the output layer, and superscript $k$ corresponds to the $k$-th output neuron (representing the $k$-th class in soft-max classification).
Although Rank-1 FNN looks similar to a conventional FCFNN, it is significantly different due to the constraint  in equation (\ref{eq:rank_1_decomposition}).
Constraint (\ref{eq:rank_1_decomposition}) considerably reduces the number of trainable parameters to $Q\sum_{d=1}^D I_d + QC$, whereas the number of parameters for the FCFNN is $Q\prod_{d=1}^D I_d + QC$. 

\subsection{Rank-$R$ FNN Modeling}
The strength of Rank-1 FNN lies in the  reduction of the number of trainable parameters compared to an ordinary FCFNN. Nevertheless, this reduction also affects its representation power, i.e., it limits its ability to model complex statistical relations between the input and the output variables\footnote{We will see in Section \ref{sec:theory} that Rank-$R$ FNNs are universal approximators, while Rank-1 FNNs are not.
	This is a crucial difference between the class of functions implemented by Rank-1  and Rank-$R$ FNNs.}.
To address this limitation, we
move to higher rank decompositions
and propose Rank-$R$ FNN,
a neural network whose weights, $\bm W^{(q)}$ connecting the input to the $q$-th neuron of the hidden layer, satisfy a rank-$R$ CP decomposition:
\begin{equation}
	\label{eq:rank_r_decomposition}
	\bm W^{(q)} = [\![\bm W_1^{(q)}, \cdots ,\bm W_D^{(q)} ]\!] \in \mathbb R^{I_1 \times \cdots \times I_D},
\end{equation}
or else
\begin{equation}
	\label{eq:rank_r_decomposition_2}
	\text{vec}(\bm W^{(q)}) = (\bm W_D^{(q)} \odot \cdots \odot \bm W_1^{(q)}) \bm 1_R \in \mathbb R^{\prod_{d=1}^{d=D}I_d},
\end{equation}
where $\bm 1_R$ stands for a vector with $R$ ones. 
The total number of weights of a Rank-$R$ FNN
is $RQ\sum_{d=1}^D I_d + QC$.
Under the CP decomposition constraint,
the output of the $q$-th hidden neuron becomes
\begin{equation}
	\label{eq:rank_r_weights}
	u_q = g(\langle \bm W^{(q)}, \bm X \rangle).
\end{equation}
Based on equations (\ref{eq:1}), (\ref{eq:rank_r_decomposition}) and (\ref{eq:rank_r_decomposition_2}),
it holds true that
\begin{equation}
	\label{eq:observation1}
	\langle \bm W^{(q)}, \bm X \rangle = 
	\text{diag}\langle \bm W_d^{(q)}, \bm Z_{\neq d}^{(q)} \rangle, 
\end{equation}
where 
\begin{equation}
	\label{eq:observation2}
	\begin{split}
		\bm Z_{\neq d}^{(q)} = 
		\bm X_{(d)}( \bm W_D^{(q)}\odot & \cdots \odot \bm W_{d+1}^{(q)}\odot \\ &\bm W_{d-1}^{(q)}\odot \cdots \odot \bm W_1^{(q)}).
	\end{split}
\end{equation}

Tensor $\bm X_{(d)}$ denotes the mode-$d$ matricization of tensor $\bm X$.  
In light of equation (\ref{eq:observation1}), the output of the $q$-th hidden neuron $u_q$ can be written as
\begin{equation}
	\label{eq:hidden_output}
	u_q = g(\langle \bm W^{(q)}, \bm X \rangle) =  \text{trace}\Big(g\big( (\bm W_d^{(q)})^T \bm Z_{\neq d}^{(q)}\big)\Big),
\end{equation}
while the output of Rank-$R$ FNN for the $k$-th class is given by (\ref{eq:rank_1_output}).
Note that  $\bm W_d^{(q)} \in \mathbb R^{I_d \times R}$ and $\bm Z_{\neq d}^{(q)} \in R^{I_d \times R}$, which implies that $(\bm W_d^{(q)})^T \bm Z_{\neq d}^{(q)}$ is a square matrix in $\mathbb{R}^{R \times R}$.

Matrix $\bm Z_{\neq d}^{(q)}$ is a transformation of
input $\bm X$ and is independent from $\bm W_d^{(q)}$.
Under  the previous notation,
it becomes clear that
Eq. (\ref{eq:hidden_output}) actually resembles the operation performed by a single perceptron with weights $\bm W_d^{(d)}$ and input matrix $\bm Z_{\neq d}^{(q)}$.
If the rank-$R$ canonically decomposed weights $\bm W_{d'}^{(q)}$,
$d' \neq d $ are known, then input matrices $\bm Z_{\neq d}^{(q)}$ are known too.
The previous observation underpins the derivation of
the optimization algorithm presented in the next subsection which estimates the weights of Rank-$R$ FNN models.

\begin{figure}[t]
	{\begingroup
		\removelatexerror
		\begin{algorithm}[H]
			\caption{Rank-$R$ FNN weights estimation}
			\label{alg:2}
			\SetAlgoLined
			\textbf{Initialization:}\\
			1. Set Iteration Index $n\rightarrow 0$\\
			2. Initialize all  weights $\bm W_d^{(q)}(n)$ and $\bm v^{(k)}(n)$ \\
			for $d=1,...,D$, $q=1,2,\cdots,Q$, $k=1,...,C$  \\
			3. \Repeat{termination criteria are met}{
				\For{$d=1,...,D$}{
					Consider $\bm W_{d'}^{(q)}$ fixed, for $d' \neq d, d' \in \{1, \dots, D\}$ \\
					\For{$q=1,...Q$}{
						3.1 Estimate  transformed input matrix $\bm Z_{\neq d}^{(q)}$ [see Eq. (\ref{eq:observation2})] \\
						3.2 Compute Rank-$R$ FNN output and loss [see relations (\ref{eq:hidden_output}) and (\ref{eq:neg_log_likelihood})] \\
						3.3 Update  weights $\bm W_d^{(q)}(n)$ towards  negative direction of $\partial L / \partial \bm W_d^{(q)}$ 
					}	
				}
				\For{$k=1,...,C$}{
					3.4 Update  weights $\bm v^{(k)}(n)$ towards  negative direction of $\partial E / \partial \bm v^{(k)}$
				}
				
				Set $n \rightarrow n+1$
			}
		\end{algorithm}
		\endgroup}
\end{figure}

\subsection{Estimation of Rank-$R$ FNN Weights}
Let us aggregate all weight parameters of Rank-$R$ FNN as 
\begin{equation}
	\label{eq:rank-r-weights}
	\bm V = \{\bm v^{(k)}\}_{k=1}^C  \:\:,
	\:\: \bm W_d = \{\bm W_d^{(q)}\}_{q=1}^Q,
\end{equation}
for $d=1,\cdots,D$. In (\ref{eq:rank-r-weights}) $\bm v^{(k)} $ collects the weights between the hidden and the output layer, and superscript $k$ corresponds to the $k$-th output neuron.
Given training data  $\mathcal D =\{(\bm X_i, \bm t_i)\}_{i=1}^N$
and sets $\{\bm W_d\}_1^D$, $\bm V$,
we use the negative log-likelihood  function
\begin{equation}
	\label{eq:neg_log_likelihood}
	\begin{split}
		L(\bm W_1,..., & \bm W_D, \bm V ;\mathcal D) = \\ & -\sum_{i=1}^N \sum_{k=1}^{C} t_{i,k}\log p^k(\bm X_i; \{\bm W_d\}_1^d, \bm V)
	\end{split}
\end{equation}
to asses the classification performance of the corresponding Rank-$R$ FNN model on training data.
Among all possible Rank-$R$ FNN models whose weights satisfy
the CP decomposition constraint in~(\ref{eq:rank_r_decomposition}),
we opt for that ones which achieve minimal
value $L(\bm W_1,...,\bm W_D, \bm V ;\mathcal D)$.

Equations (\ref{eq:observation2}) and (\ref{eq:hidden_output}) are crucial for the
derivation of an alternating optimization algorithm to minimize the objective function in (\ref{eq:neg_log_likelihood}). When matrices $\bm V$ and $\bm W_{d'}^{(q)}$,
$d' \neq d$,  are known, matrix $\bm Z_{\neq d}^{(q)}$ is reconstructed as in equation~(\ref{eq:observation2}),
which implies that the only unknown parameter to be estimated is the weight matrix $\bm W_{d}^{(q)}$.
We adopt a coordinate descent minimization scheme for the
minimization of ($\ref{eq:neg_log_likelihood})$
where in each step we keep $\bm V$ and $\bm W_{d'}^{(q)}$ fixed
and iterating over all $d\neq d'$, $d \in \{1,\dots, D\}$
we minimize with respect to $\bm W_{d}^{(q)}$.
The derivative
$\partial L / \partial \bm W_d^{(q)}$ can be computed by the backpropagation algorithm and the estimation of Rank-$R$ FNN weights is done with gradient descent steps,
see Algorithm \ref{alg:2}.

\section{Theoretical Properties of Rank-$R$ FNN}
\label{sec:theory}
In this Section, we prove learnability and
universal approximation properties for the class of functions implemented by Rank-$R$ FNN models.
For this purpose,
we reveal their connections to ordinary FCFNN models
and,
in Theorem \ref{thm:rank_fnn_theorem} below,
we construct a subjective mapping
between Rank-$R$ FNN  and ordinary FCFNNs.

\begin{theorem}[Rank-$R$ FNN Theorem]
	\label{thm:rank_fnn_theorem}
	Let $f: \mathbb R^{\prod_{d=1}^D p_d}\rightarrow \mathbb F$ be a two-layer fully connected FNN with $Q$ hidden neurons, that maps the vectorized form, $\text{vec}(\bm A)$, of a tensor object $\bm A \in \mathbb R^{p_1\times\cdots\times p_D}$ to $\mathbb F$. If $p_d < \infty$, $d=1,\cdots, D$, then there exists a Rank-$R$ FNN $g: \mathbb R^{p_1 \times \cdots \times p_D}\rightarrow \mathbb F$ with $Q$ hidden neurons that is equal to $f$.
\end{theorem}

\begin{proof}
	See Appendix A.
\end{proof}

\noindent \textit{Remark 1}:
The Rank-$R$ Theorem holds both for regression and classification tasks. In case of regression tasks $\mathbb F \equiv \mathbb R$, whereas for classification tasks $\mathbb F$ is the set of available classes.

An immediate consequence of Rank-$R$ FNN Theorem
is Corollary~\ref{cor:univ_approx} below, 
which is based on the fact that two-layer sigmoid FCFNNs are universal approximators,
i.e.,
given any continuous function $h$ defined on a compact subset $\mathcal S$ of $\mathbb R^n$ and any $\epsilon>0$, there exists a two layer FCFNN implementing
a function that is within $\epsilon$ of $h$ at each point of $\mathcal S$ (see \cite{cybenko1989approximation}, \cite{hornik1990universal} or Exercise 20.1 in~\cite{shalev2014understanding}).

\begin{corollary}
	\label{cor:univ_approx}
	Rank-$R$ FNNs are universal approximators; given any continuous function $h$ defined on a compact subset $\mathcal S$ of $\mathbb R^n$, there is a Rank-$R$ FNN that implements a function which is arbitrarily close to $h$ at each point in $\mathcal S$.
\end{corollary}

Apart from the universal approximation property,
another consequence of Theorem~\ref{thm:rank_fnn_theorem} is that
the class of Rank-$R$ FNNs,
with fixed number of hidden neurons,
is a {\it learnable} class of functions.
Indeed,
for any  fixed $R<\infty$, all Rank-$R$ FNN models, with rank lower
that $R$, map to
ordinary FCFNNs with weights in $\mathbb{R}^{\prod_{d=1}^D p_d}$.
Since the class of FCFNNs with finite dimensional weights and fixed number of neurons has finite VC dimension,
the class of Rank-$R$ FNNs can be efficiently learned
by the empirical minimization principle in polynomial time,
see also   \cite{blumer1989learnability}, Theorem 20.4 in \cite{shalev2014understanding},
or Exercise 20.5 in~\cite{shalev2014understanding}. This lead us to Corollary~\ref{cor:learnability}.

\begin{corollary}
	\label{cor:learnability}
	The class of functions defined by Rank-$R$ FNN,
	with a fixed number of hidden neurons,
	has finite sample complexity and, thus, is learnable with the empirical risk minimization
	principle.
\end{corollary}

Summarizing, for any FCFNN, with fixed number of hidden neurons, there exists a Rank-$R$ FNN with the same number of  neurons that reproduces exactly the same output. This form of equivalence implies that Rank-$R$ FNNs exhibit the universal approximation property when the number of hidden neurons grows unbounded. Finally, the class of functions implemented by Rank-$R$ FNNs of fixed rank, finite number hidden neurons,
and no cycles or loops in their graph, is learnable. Next we proceed with the experimental validation of Rank-$R$ FNNs.

\section{Experimental Validation}
In this Section, we evaluate the classification performance of Rank-$R$ FNN\footnote{The Rank-$R$ FNN code used in these experiments is available in Python at https://github.com/konstmakantasis/Rank-R-FNN} using hyperspectral imagery, which is a typical example of high-order data. Three widely known and publicly available datasets, captured by three different sensors, are used. In particular, we use i) the Indian Pines dataset, which has been captured by AVIRIS sensor and consists of 224 spectral bands and 10,249 labeled pixels assigned to 16 different classes, ii) the Pavia University dataset, which has been captured by ROSIS sensor and consists of 103 spectral bands and 42,776 labeled pixels assigned to 9 different classes, and iii) the Botswana dataset, which has been captured by Hyperion sensor and consists of 145 spectral bands and 3,248 labeled pixels assigned to 14 different classes. Figure \ref{fig:1} presents the employed datasets along with their ground truth.

We compare the performance of Rank-$R$ FNN against the 
CNN model of \cite{makantasis2015deep} for two main reasons.
First, the CNN of \cite{makantasis2015deep} is a benchmarking model for hyperspectral image classification. Second, it is a simple yet very efficient architecture, although it is not designed explicitly for hyperspectral data classification. In other words, the proposed Rank-$R$ models and the CNN mentioned above do not employ any specific design choices for exploiting the particular characteristics of hyperspectral data. Thus they can be used for any image pixel classification task.

At this point, we should highlight that Rank-$R$ FNN models can be straightforwardly applied on data represented as tensors of arbitrary order. For example, hyperspectral data that have been enhanced with mathematical morphology features \cite{jouni2019hyperspectral}  are usually represented as tensor objects of order larger than three. On the contrary, efficient processing of high-order data with sophisticated CNN models cannot be done in a straightforward manner since the high dimensionality of such data practically renders their application extremely inefficient due to the high computational cost of the convolution operation in more than three dimensions. 

We set the number of hidden neurons of both Rank-1 FNN and Rank-$R$ FNN to 75, while the employed CNN consists of two convolutional layers with 150 and 300 kernels, respectively, of dimension $3\times3$, and a fully connected layer with 75 hidden neurons. For all learning models we use the same training and testing datasets.  Moreover, we investigate the robustness of these models under different levels of white noise.

\begin{figure}[t]
	\begin{minipage}{\linewidth}
		\centering
		\centerline{\includegraphics[width=1.0\linewidth]{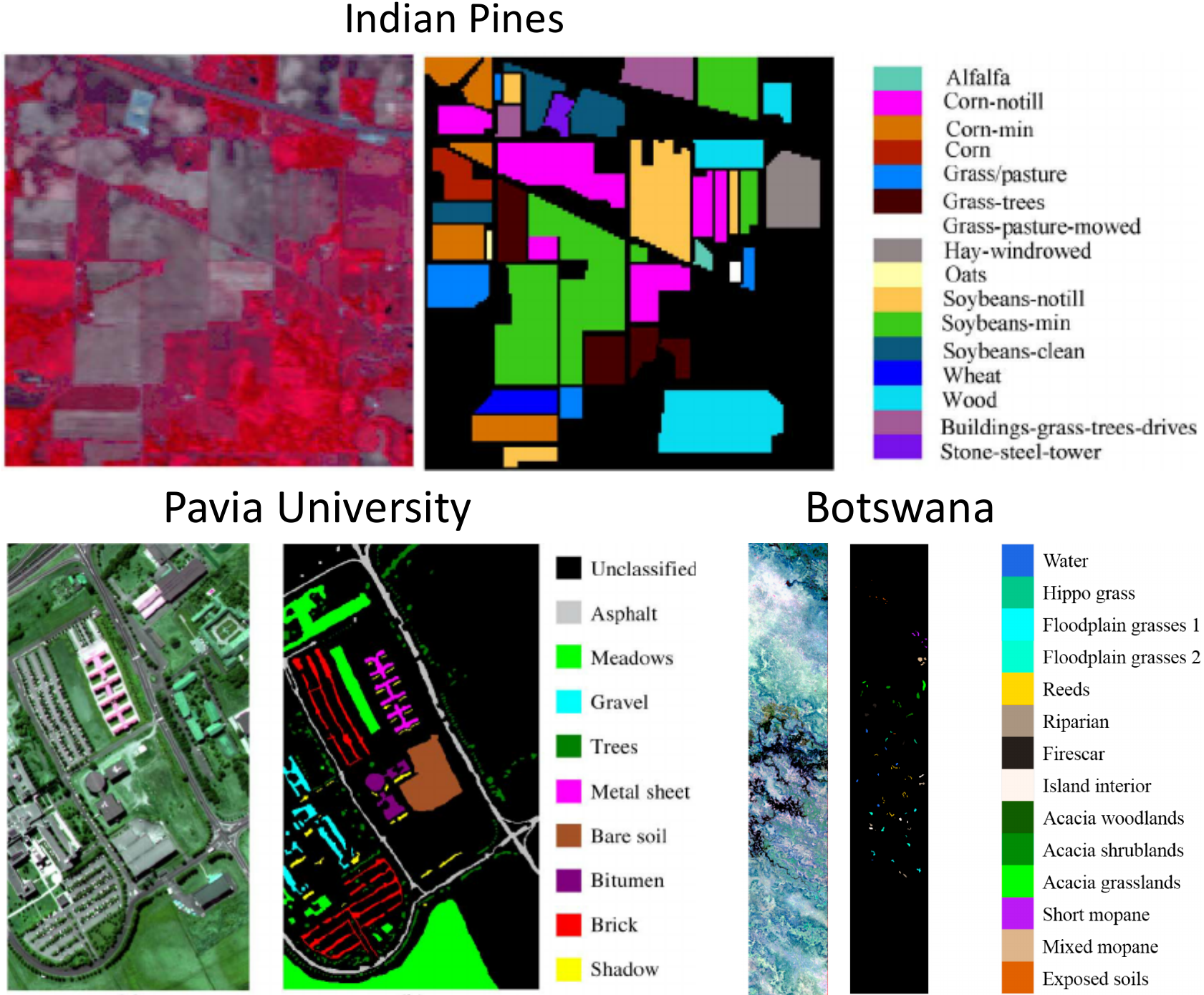}}
	\end{minipage} 
	\caption{The three hyperspectral datasets employed in the current study. Ground truth for each one of the datasets is also presented.}
	\label{fig:1}
\end{figure}

\subsection{Dataset Description}
A hyperspectral image is a 3-order tensor of dimensions $p_1\times p_2 \times p_3$, where $p_1$ and $p_2$ correspond to  height and width of the image, while $p_3$ corresponds to its spectral bands.
To conduct pixel-wise classification, i.e., to classify each pixel $I_{x,y}$ at location $(x,y)$ according to the material it depicts, we follow the approach proposed in \cite{makantasis2018tensor}. Specifically, it is assumed that the label of a square patch $\bm X_{x,y}$ of size $s \times s \times p_3$ centered at $(x,y)$ has the same label with pixel $I_{x,y}$. Denoting as $\bm t_{x,y}$ the ground truth label of $I_{x,y}$, we form the dataset $\mathcal D=\{(\bm X_{x,y}, \bm t_{x,y})\}$ for training and evaluation purposes. In all experiments we set parameter $s$ equal to 5. That way we exploit spatial information of pixels, and, at the same time, satisfy the assumption that, for the majority of pixels the square patch $\bm X_{x,y}$ has  same label as $I_{x,y}$ \cite{chen2014deep}. 

\begin{figure*}[t]
	\includegraphics[scale=0.55]{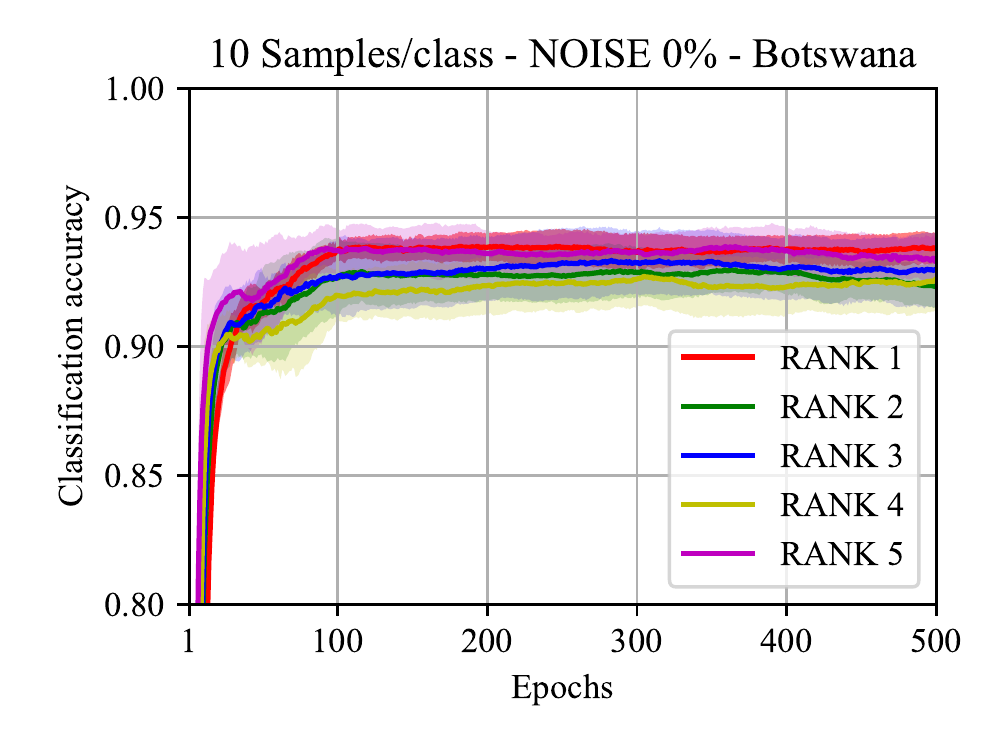} 
	\includegraphics[scale=0.55]{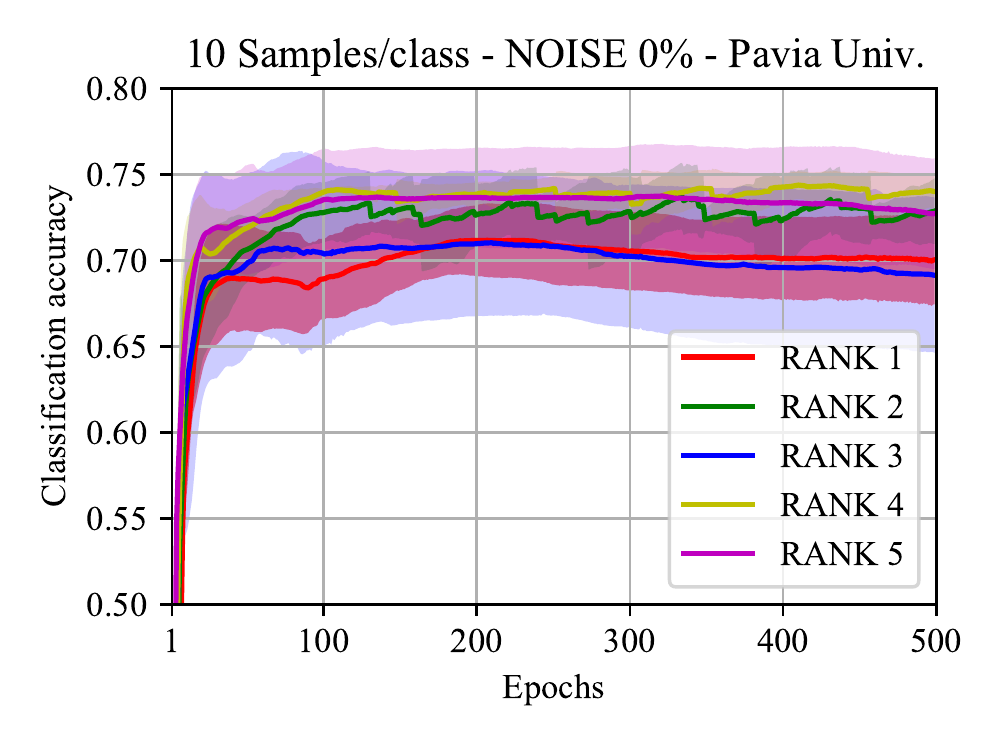} 
	\includegraphics[scale=0.55]{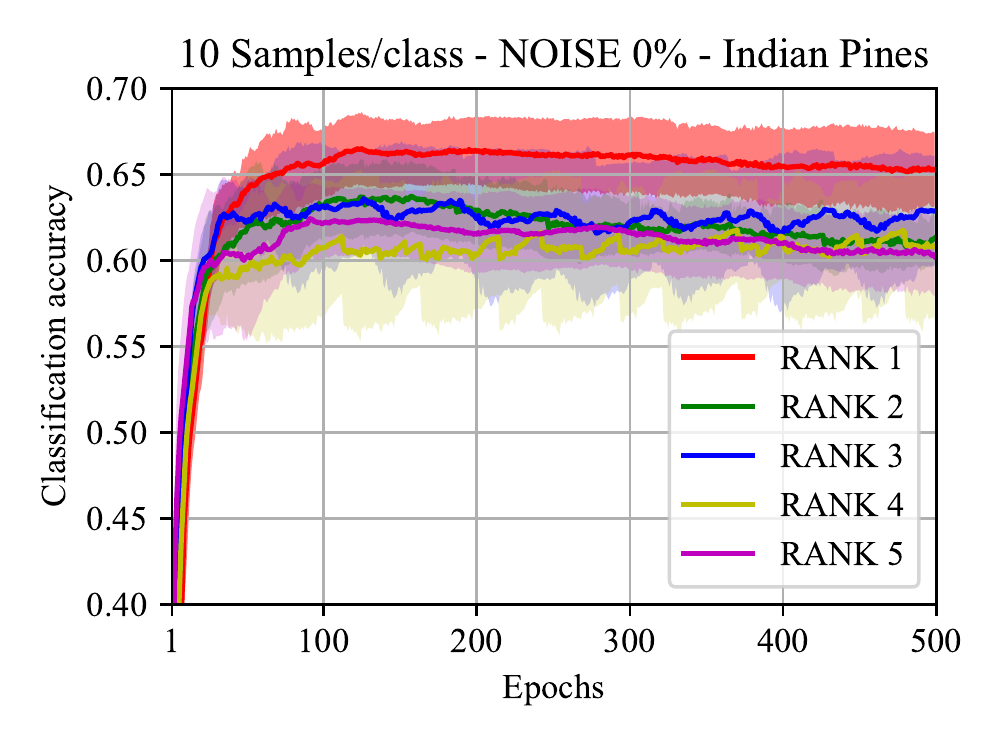}
	\caption{The classification accuracy and standard deviation on the test set of the proposed Rank-$R$ model versus different epochs.}
	\label{fig:RANK}
\end{figure*}
\begin{figure*}[t]
	\includegraphics[scale=0.55]{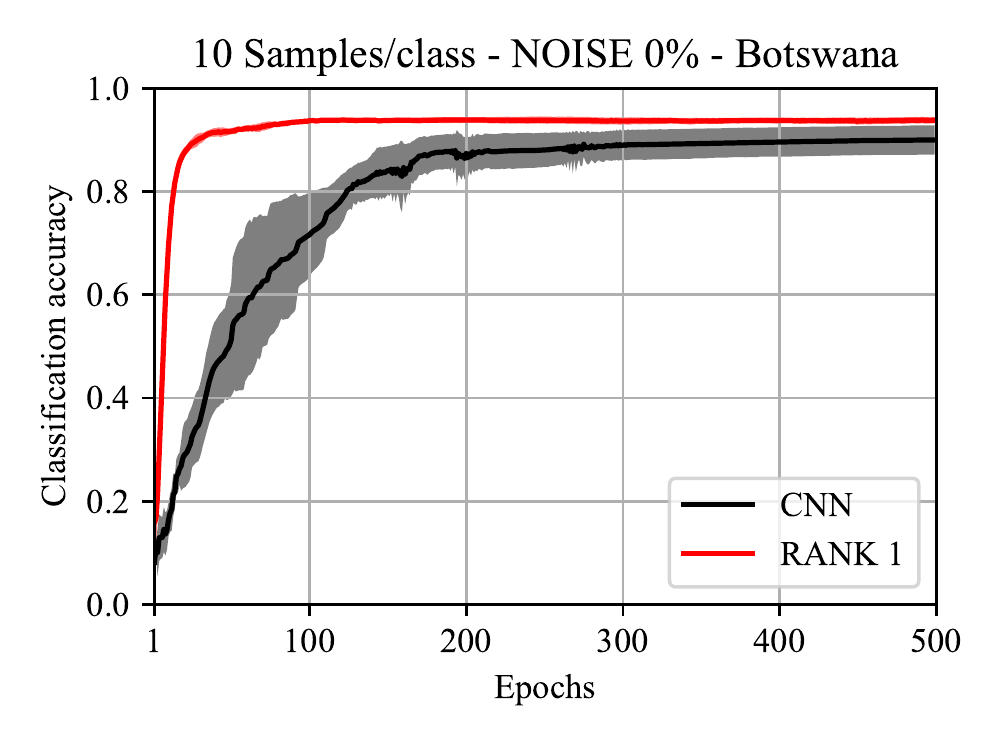} 
	\includegraphics[scale=0.55]{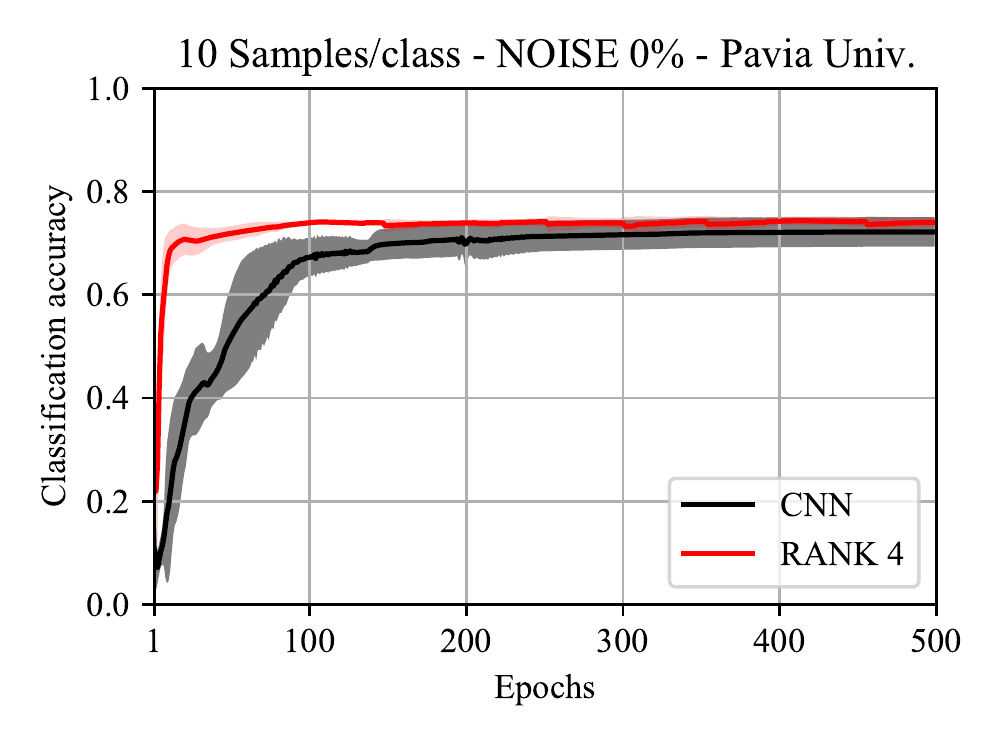} 
	\includegraphics[scale=0.55]{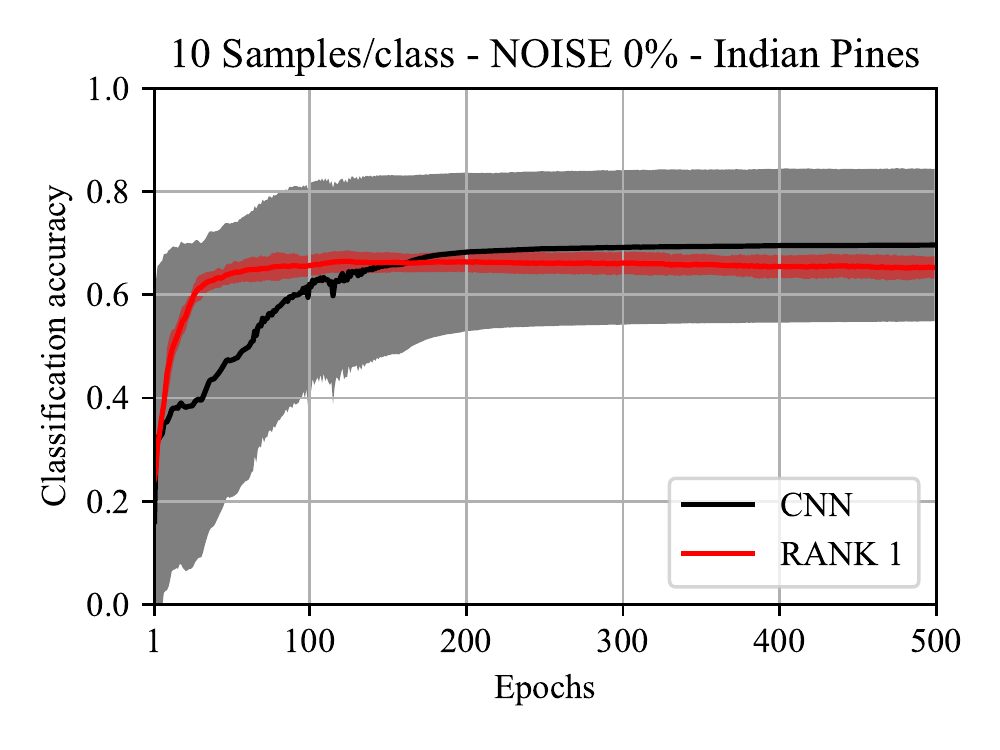}
	
	\includegraphics[scale=0.55]{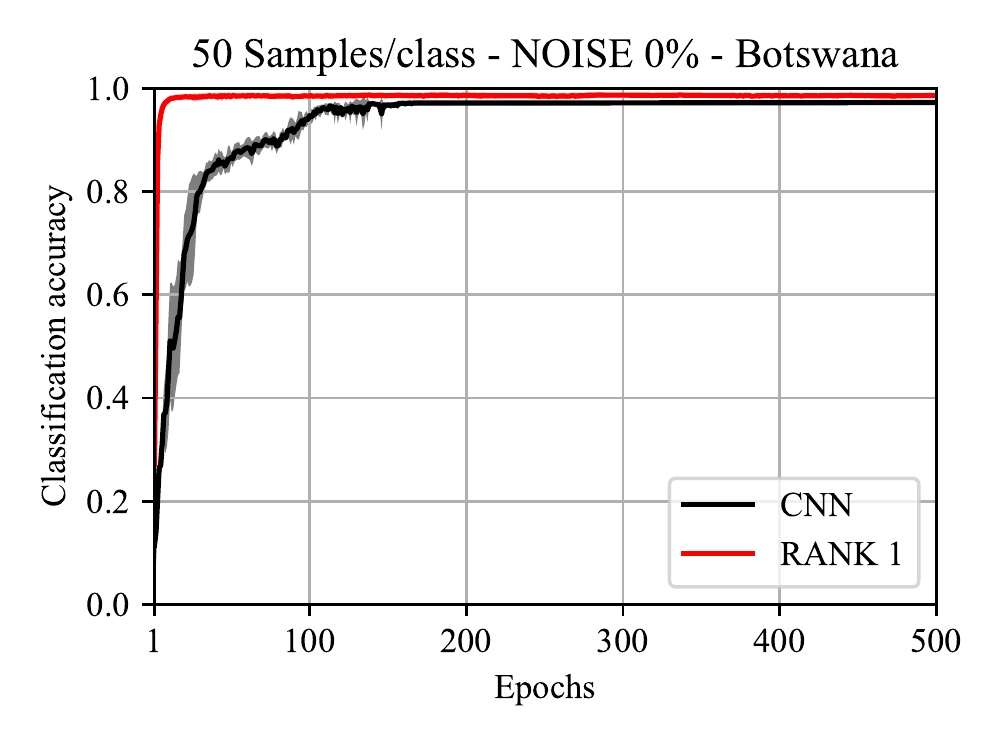} 
	\includegraphics[scale=0.55]{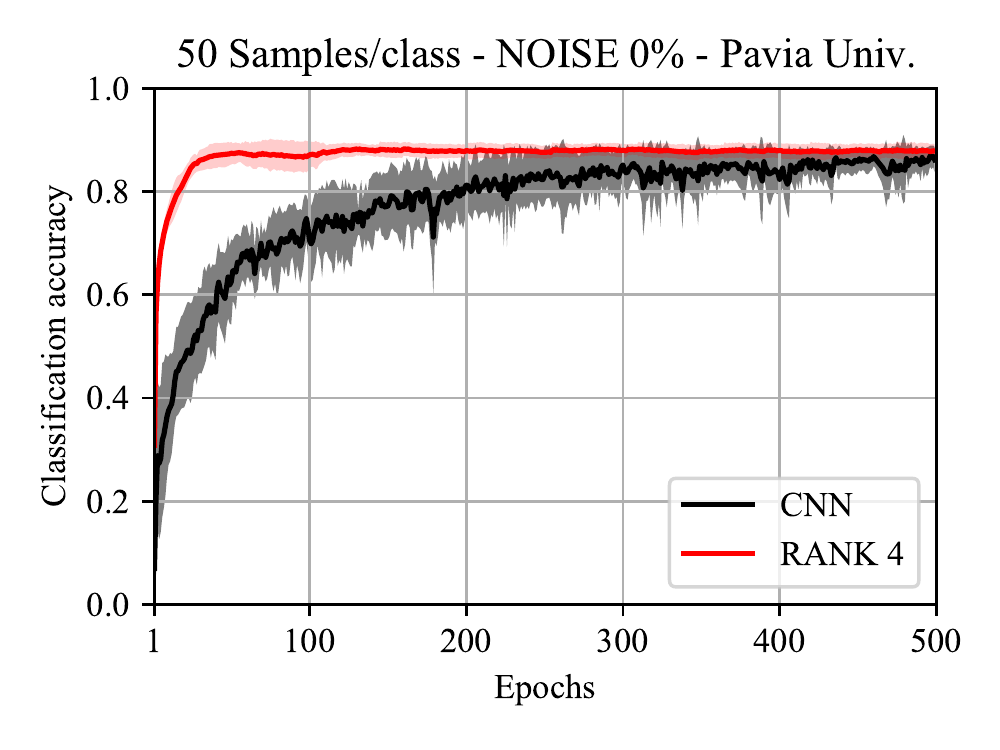} 
	\includegraphics[scale=0.55]{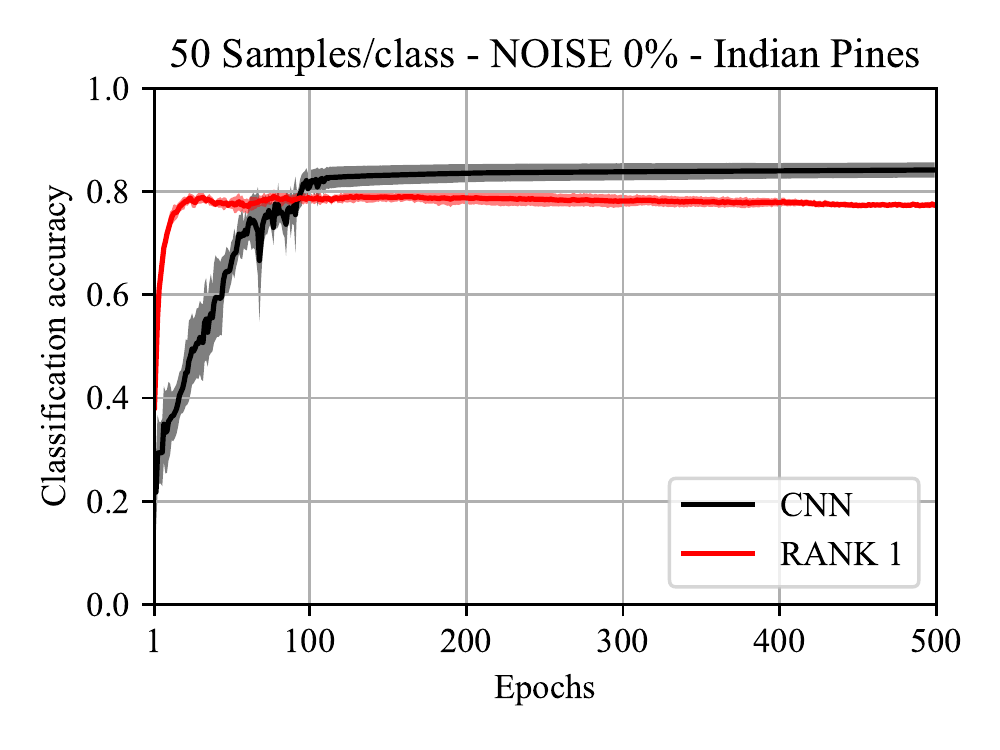} \
	\caption{Comparative performance as far as classification accuracy on test set is concerned between the proposed RANK-$R$ FNN model and the state of the art CNN based network of \cite{makantasis2015deep}. (Top) The case of $\alpha=10$ samples per training class for the Botswana, Pavia University and the Indian Pines dataset respectively. (Bottom) The case of $\alpha=50$ samples per training class for the Botswana, Pavia University and the Indian Pines dataset respectively. }
	\label{fig:RANK_CNN}
\end{figure*}

Although conventional deep learning models (such as the CNN model of \cite{makantasis2015deep}) can achieve almost perfect classification results for these datasets when the number of training samples is adequately high, we choose to train the models using a limited number of training samples in this study. This way, we focus on models' capacity to learn small sample setting classification tasks since employing a small number of training samples is a common limitation in many real-world applications such as hyperspectral image classification.

For this reason, in our experiments, we vary the number of samples per class used for training to evaluate the proposed tensor-based model's performance when the number of training data is limited. In particular, we randomly select a specific number of samples $\alpha$ per class for training, while the rest are used for testing purposes. In our case, the samples per class used for training are $\alpha=10$ and $50$. If some class includes fewer samples, we select a portion of $50\%$ randomly for training.

We also investigate the robustness of the models to the presence of noise. In particular, in our experiments, we have added a white noise level of 20\% in every band of each pixel. Due to the involved randomness during training/testing sets splitting, we conduct each experiment for ten runs and report the average classification accuracy and standard deviation on testing sets

In this study, we opt for the holdout cross-validation scheme, instead of the powerful K-fold cross-validation for the following reason. Although, K-fold cross-validation is a widely used scheme for evaluating the performance of learning models and approximating their true error, in our case it is not applicable. This is due to the fact that the classes are not equally represented in the datasets, and thus it is not possible to use K-fold cross-validation (even with different K for each dataset) to evaluate the performance of the models.

\subsection{Results}
In this section, we present the performance of the proposed Rank-$R$ FNN models and we compare them with the state-of-the-art CNN model of \cite{makantasis2015deep} on the employed datasets. In the first subsection, we present and discuss the performance of the different models when they are trained on noise-free data. In the second, we investigate the robustness of the same models by evaluating their performance of noisy data.

\subsubsection{Models' performance evaluation on noisy-free data}

For  Rank-$R$ FNN we used five different values for $R \in \{1, 2, 3, 4, 5\}$ and denote the respective classifiers as Rank-1 FNN, Rank-2 FNN, Rank-3 FNN, Rank-4 FNN, and Rank-5 FNN. 

Figure \ref{fig:RANK} illustrates the classification accuracy versus the number of training epochs of the proposed Rank-$R$ FNN classifiers over the three examined data sets; \textit{the Botswana, Pavia University and the India Pines datasets}. In this figure, we have also shown the standard deviation of the classification accuracy results obtained over the 10 different runs as an area of the same colour around the average line. As is observed, all the Rank-$R$ models converge quickly withing a few training epochs; less than 20 for all cases. The degree of the rank decomposition slightly affects the performance, and it is usually application dependent. In particular, for Botswana and Indian Pines datasets, the best performance is achieved for $R=1$. However, for Pavia University, the best performance is for $R=4$.

Figure \ref{fig:RANK_CNN} depicts comparisons between the proposed Rank-$R$ FNN model and the state-of-art CNN network of \cite{makantasis2015deep}. The results have been presented versus the number of training epochs and two different numbers of $\alpha$. In particular, Figure \ref{fig:RANK_CNN}(top) indicates the case for  $\alpha=10$ samples per class, while Figure \ref{fig:RANK_CNN}(bottom) for $\alpha=50$ samples per class. In this figure, $R=1$ for the Botswana and India Pines datasets and $R=4$ for the Pavia University since these values give the best classification accuracy (see Figure \ref{fig:RANK}). As is observed, in all cases, the proposed Rank-$R$ model converges more rapidly than the conventional CNN-based network of \cite{makantasis2015deep}. Besides, the proposed Rank-$R$ model presents a much smaller standard deviation of the average classification accuracy in all cases, indicating the robustness of our model against different execution runs. As the number of training samples decreases, the proposed Rank-$R$ FNN model's performance remains robust with minimal deviations from the average classification accuracy over the 10 different runs. Indeed, for smaller number of samples, better improvement is achieved by our proposed model compared to CNN. Besides, the CNN model's standard deviation increases, especially in the case of the Indian Pines dataset, as a small number of training samples per class is selected.  

Table \ref{table:1} presents the average classification accuracy and the respective standard deviation on the test sets over the three examined datasets. The results have been obtained for $R \in \{1, 2, 3, 4, 5\}$ and compared with the CNN model of \cite{makantasis2015deep}. In this table, we have depicted the results for 50 and 500 epochs, respectively.  For all models, we have selected  $\alpha =10$ samples per class. The latter is selected to indicate the performance of the proposed Rank-$R$ FNN model in case a few training samples are employed.

As is observed, in Botswana and Pavia University datasets, the proposed Rank-$R$ FNN model is about 3.8\% and 1.8\% respectively better in performance, while in the case of the Indian pines dataset is worse about 4.3\%. However, CNN's standard deviation is times larger than the standard deviation of the proposed Rank-$R$ FNN. In particular, for the Pavia University dataset, the standard deviation of our model is $1.8$ ( $1.02$ vs $2.85$) times smaller than of CNN, while for Botswana is $3.6$ ($0.6$ vs $2.81$) times smaller and for India pines of about $6.0$ ($2.12$ vs $14.8$) times smaller (Table. \ref{table:1}).

Table \ref{table:2} presents the number of trainable parameters in each model. Specifically, the CNN employs \textit{41, 48} and \textit{57 times} more parameters than the Rank-1 FNN for the Indian Pines, the Pavia University and the Botswana datasets, respectively. This means that our Rank-$R$ FNN requires times smaller number of parameters for learning the classification task.

\begin{table}[t]
	\centering
	\caption{Average classification accuracy and standard deviation (\%) on test set in case of no noise and 10 samples per class.}
	\newcolumntype{L}[1]{>{\hsize=#1\hsize\raggedright\arraybackslash}X}%
	\newcolumntype{C}[1]{>{\hsize=#1\hsize\centering\arraybackslash}X}%
	\label{table:1}
	
	\begin{tabularx}{\linewidth}{L{5.8}C{6.4}C{5.4}C{5.4}}
		\hline \hline 
		& Indian Pines & Botswana & Pavia Uni. \\ \hline \\
		EPOCHS=50\\ \\
		Rank-1 FNN   & $64.24 \pm 2.00$ & $ 91.73 \pm 0.61$ & $68.92   \pm 3.00 $\\ \hline
		Rank-2 FNN   & $62.02  \pm 3.25$ & $91.36  \pm 1.76$ & $70.66\pm 1.88$\\ \hline
		Rank-3 FNN   & $62.36  \pm 2.30$ & $91.50 \pm 1.23$ & $69.90  \pm 5.16$\\ \hline 
		Rank-4 FNN   & $59.68  \pm 3.97$ & $90.59  \pm 1.34$ & $71.93  \pm 1.43$\\ \hline 
		Rank-5 FNN   & $60.31  \pm 4.36$ & $92.14   \pm 1.79$ & $72.37  \pm 3.20$\\ \hline
		CNN  		 & $47.36  \pm 26.5$ & $54.03  \pm 13.2$ & $52.36 \pm 10.4$\\ \hline \hline \\
		EPOCHS=500\\ \\
		Rank-1 FNN   & $65.28  \pm 2.12$ & $93.80   \pm 0.60$ & $70.03 \pm 2.59$\\ \hline
		Rank-2 FNN   & $61.29  \pm 1.55$ & $92.32 \pm 0.83$ & $72.86   \pm 1.91$\\ \hline
		Rank-3 FNN   & $62.87 \pm 3.24$ & $92.97 \pm 1.44$ & $69.13   \pm 4.52$\\ \hline 
		Rank-4 FNN   & $60.80   \pm 4.22$ & $ 92.57 \pm 1.18$ & $74.00  \pm 1.02$\\ \hline 
		Rank-5 FNN   & $60.23   \pm 2.32$ & $93.33  \pm 0.86$ & $72.72  \pm 3.22$\\ \hline
		CNN  		 & $69.62  \pm 14.8$ & $89.99  \pm 2.81$ & $72.18  \pm 2.85$\\ \hline \hline\\
		
	\end{tabularx}
\end{table}

\begin{table}[t]
	\centering
	\caption{Number of trainable parameters for each model.}
	\newcolumntype{L}[1]{>{\hsize=#1\hsize\raggedright\arraybackslash}X}%
	\newcolumntype{C}[1]{>{\hsize=#1\hsize\centering\arraybackslash}X}%
	\label{table:2}
	
	\begin{tabularx}{\linewidth}{L{5.8}C{6.4}C{5.4}C{5.4}}
		\hline \hline 
		& Indian Pines & Botswana & Pavia Uni. \\ \hline
		
		Rank-1 FNN   & $\sim 17K$  & $\sim 13K$  & $\sim 10K$\\ \hline
		Rank-2 FNN   & $\sim 33K$  & $\sim 24K$  & $\sim 18K$\\ \hline
		Rank-3 FNN   & $\sim 49K$  & $\sim 36K$  & $\sim 26K$\\ \hline 
		Rank-4 FNN   & $\sim 64K$  & $\sim 48K$  & $\sim 35K$\\ \hline 
		Rank-5 FNN   & $\sim 80K$  & $\sim 59K$  & $\sim 43K$\\ \hline
		CNN  		 & $\sim 700K$ & $\sim 625K$ & $\sim 570K$\\ \hline \hline
	\end{tabularx}
\end{table}

\begin{table}[b]
	\centering
	\caption{The p-values for accepting the null hypothesis, $H_{Rank-R}$, that the Rank-$R$ FNN, for $R=1,2,3,4,5$, performs the same as the CNN in the case of no noise and 10-50 samples per class.}
	\newcolumntype{L}[1]{>{\hsize=#1\hsize\raggedright\arraybackslash}X}%
	\newcolumntype{C}[1]{>{\hsize=#1\hsize\centering\arraybackslash}X}%
	\label{table:3}
	
	\begin{tabularx}{\linewidth}{L{5.8}C{6.4}C{6.4}C{6.4}}
		\hline \hline 
		& Indian Pines & Botswana & Pavia Uni. \\ \hline
		Samples/Class& 10 \:\:\:\:\:\:\:\:\:\: 50 & 10 \:\:\:\:\:\:\:\:\:\: 50 & 10 \:\:\:\:\:\:\:\:\:\: 50 \\ \\
		EPOCHS=50 \\ \\
		
		$H_{Rank-1}$   & $0.072 - 0.002$ & $0.004 - 0.006$ & $0.030 - 0.001$\\ \hline
		$H_{Rank-2}$   & $0.072 - 0.001$ & $0.004 - 0.001$ & $0.023 - 0.001$\\ \hline
		$H_{Rank-3}$   & $0.072 - 0.003$ & $0.004 - 0.001$ & $0.024 - 0.001$\\ \hline 
		$H_{Rank-4}$   & $0.072 - 0.002$ & $0.005 - 0.001$ & $0.019 - 0.001$\\ \hline 
		$H_{Rank-5}$   & $0.072 - 0.002$ & $0.006 - 0.001$ & $0.015 - 0.001$\\ \hline \hline \\ 
		
		EPOCHS=500\\ \\
		
		$H_{Rank-1}$   & $0.201 - 0.001$ & $0.047 - 0.001$ & $0.297 - 0.562$\\ \hline
		$H_{Rank-2}$   & $0.148 - 0.001$ & $0.176 - 0.074$ & $0.701 - 0.541$\\ \hline
		$H_{Rank-3}$   & $0.499 - 0.001$ & $0.109 - 0.004$ & $0.292 - 0.785$\\ \hline 
		$H_{Rank-4}$   & $0.148 - 0.011$ & $0.148 - 0.008$ & $0.283 - 0.269$\\ \hline 
		$H_{Rank-5}$   & $0.072 - 0.005$ & $0.075 - 0.223$ & $0.801 - 0.628$\\ \hline \hline \\
		
	\end{tabularx}
\end{table}

\begin{figure*}[t]
	\includegraphics[scale=0.55]{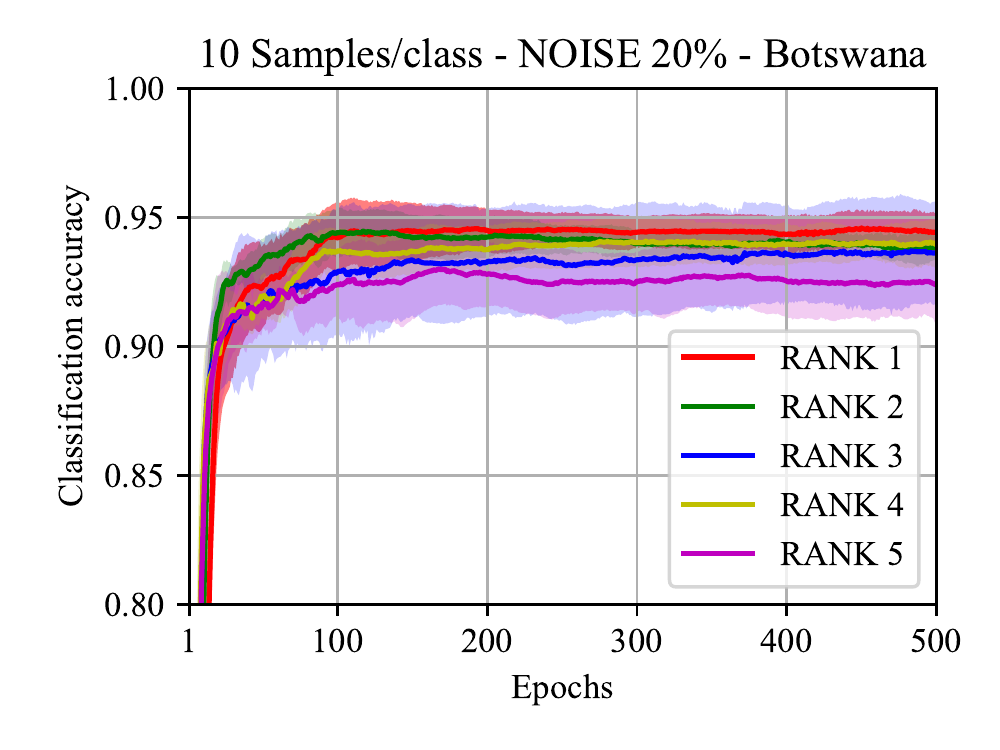} 
	\includegraphics[scale=0.55]{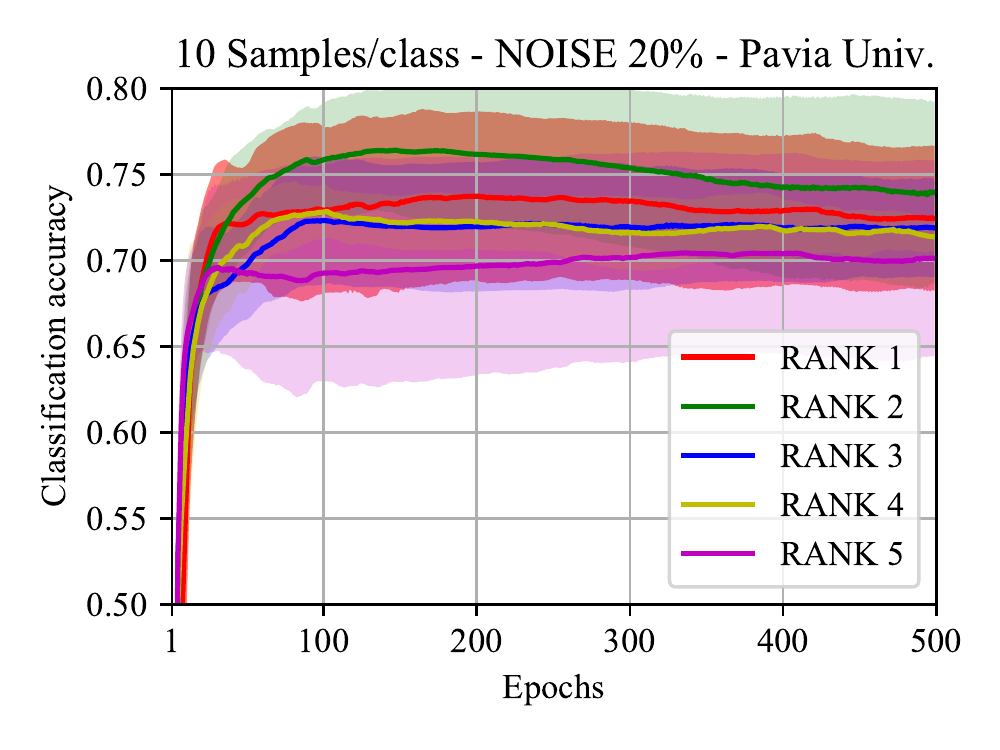} 
	\includegraphics[scale=0.55]{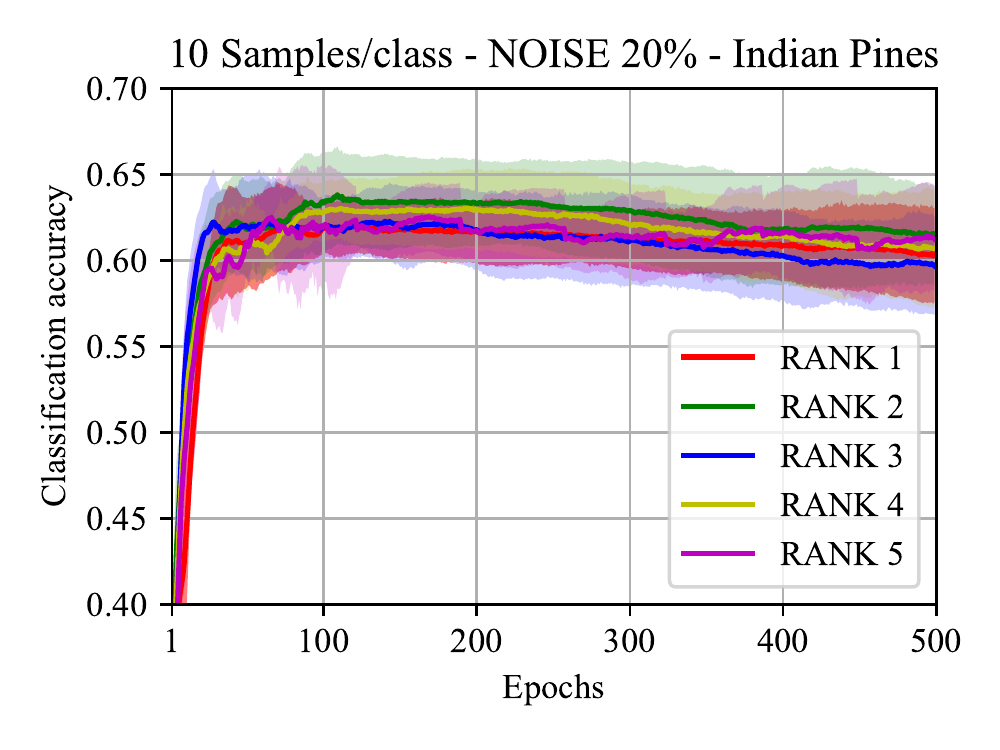}
	\caption{The performance in terms of classification accuracy and standard deviation on the test sets of the proposed Rank-$R$ model versus different epochs in case of a noise of 20\%.}
	\label{fig:RANK_NOISE}
\end{figure*}

\begin{figure*}[t]
	\hspace{0.15in}\includegraphics[scale=0.55]{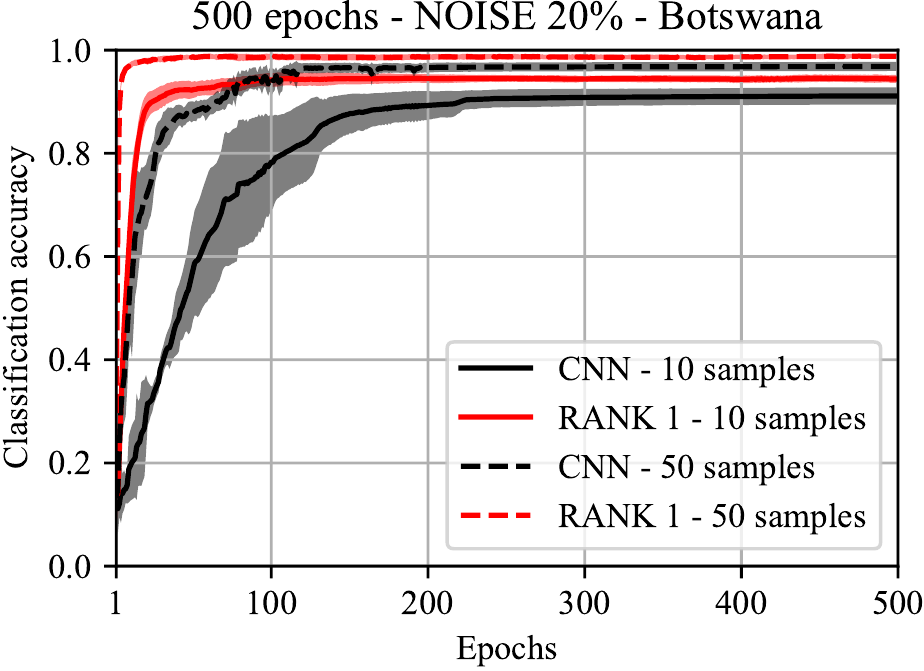} 
	\hspace{0.12in}\includegraphics[scale=0.55]{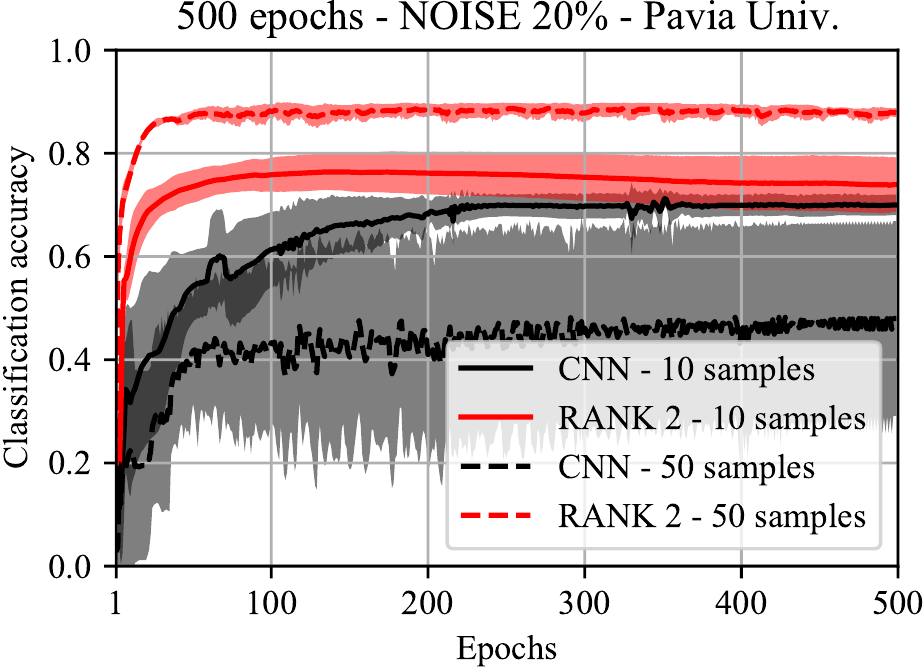} 
	\hspace{0.15in}\includegraphics[scale=0.55]{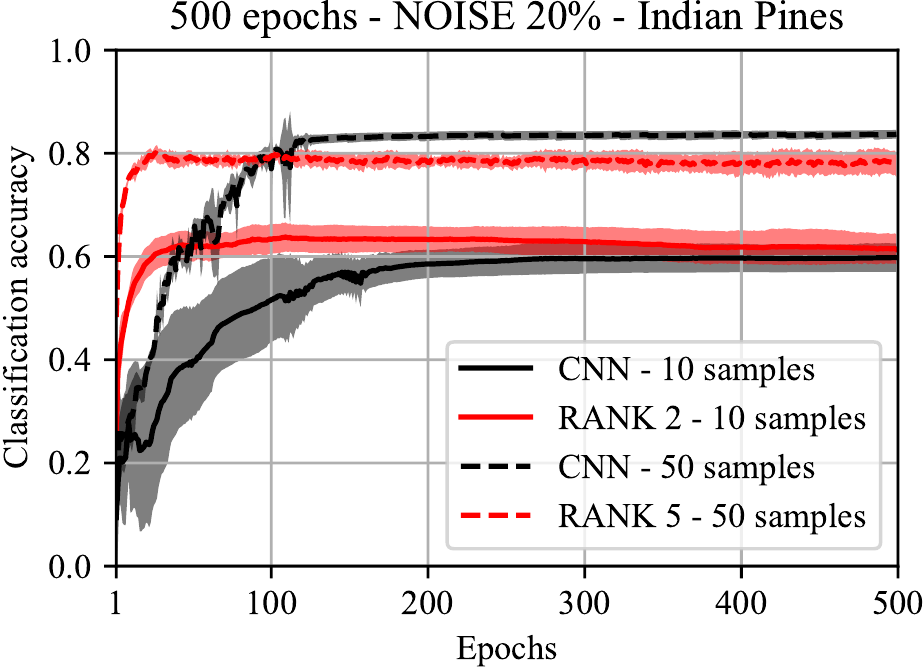}
	\caption{Comparative performance of the proposed Rank-$R$ model and the state of the art CNN-based network of \cite{makantasis2015deep}. In this figure, we assume a noise on the input data of 20\%. The comparisons have been made for $\alpha$ values of 10 and 50 respectively. In this figure, we have also depicted the standard deviation of the classification accuracy of each model. The Botswana dataset with RANK-1 since it gives the best classification accuracy (see Fig.\ref{fig:RANK_NOISE}). The Pavia University Dataset with RANK-2 since it gives the best classification accuracy (see Fig.\ref{fig:RANK_NOISE}). The Indian Pines Dataset with RANK-2 since it gives the best classification accuracy (see Fig.\ref{fig:RANK_NOISE}).}
	\label{fig:COMP_NOISE}
\end{figure*}

Finally, we conducted significance tests to check whether or not the performance of the Rank-$R$ FNN models statistically differs than the performance of the CNN. Towards this direction, we pairwise statistical tests for each dataset separately to test the null hypothesis that the CNN and each of the Rank-$R$ FNN models perform the same. First, we conducted Shapiro-Wilk tests to verify the normality of classifiers' performances. In cases where the normality assumption holds, we proceed by conducting a Levene's test to check whether the performances of the two classifiers have the same variance. Then, based on the outcome of the Levene's test, we applied the corresponding t-test to test the hypothesis that the two classifiers perform the same. In case the normality assumption is not satisfied, we proceed by conducting the non-parametric Mann-Whitney U test to test the hypothesis that the performances of the two classifiers come from the same distribution.

Based on the outcome of the t-test or the Mann-Whitney U test, we test the null hypothesis --the two models perform the same-- at the significance level 5\%. For those tests, we use the classification accuracy of the models. The sample space of the two samples (the performances of the two models that we compare) is equal to 10, that is, the number of times we repeated each experiment. We denote as $H_{Rank-R}$ the null hypothesis that the Rank-$R$ FNN models and the CNN perform the same. Table \ref{table:3} presents the results of those significance tests. For 50 epochs and 50 samples per class, all Rank-$R$ FNN models perform significantly better than the CNN on all datasets. For 50 epochs and 10 samples per class, Rank-$R$ FNN models perform significantly better than the CNN on Botswana and Pavia University datasets. For the Indian Pines, however, the difference in performance is not significant, mainly due to the large standard deviation of the CNN performance across different runs. After 500 epochs, all models have converged. For 500 epochs and 10 samples per class, the difference in models' performance cannot be considered significant, except for Rank-$1$ FNN on Botswana dataset. For 50 samples per class, Rank-$R$ FNN models (for $R=1,3,4$) perform significantly better on the Botswana dataset. For the Pavia University, the difference in performance between Rank-$R$ FNN and CNN is not significant, while for Indian Pines, CNN performs significantly better than Rank-$R$ FNN for $R=1,2,3$.  

\subsubsection{Models' performance evaluation on noisy data inputs}

In the second set of experiments, we evaluate models' performance in noisy input datasets. Figure \ref{fig:RANK_NOISE} indicates the classification accuracy on test sets versus the number of training epochs for rank $R=1,2,3,4,5$. In this figure, we have also illustrated the standard deviation of the classification accuracy over 10 different executed runs as an area of the same colour around the average line. As is observed, the best performance is achieved for $R=1$ as far as the Botswana dataset is concerned, while for the Pavia University and Indian Pines the best performance is achieved for $R=2$. Again, we observe that the proposed Rank-$R$ FNN models converge rapidly in less than 20 epochs independently of the values of $R$.

Figure \ref{fig:COMP_NOISE} depicts a comparative study of the proposed Rank-$R$ FNN model and the CNN network. This figure also illustrates the standard deviation of the classification accuracy on test sets over 10 different runs as an area around the average line. We have selected $R=1$ for the Botswana dataset, and $R=2$ for the Pavia University and Indian Pines since these rank values give the best classification accuracy (see Figure \ref{fig:RANK_NOISE}). As is observed, the proposed Rank-$R$ model presents high accuracy regardless of the noise level in its input, indicating robustness against noise. Besides, the proposed Rank-$R$ FNN model presents much lower standard deviations against different run executions than the CNN-based network of \cite{makantasis2015deep}. Also, as noisy inputs are feeding to the classification networks, the CNN model's performance more rapidly decreases compared to the proposed Rank-$R$ model. This is justified by the high representation power of CNNs, allowing them to over-fit small noise levels. 

\begin{table}[t]
	\centering
	\caption{Average classification accuracy and standard deviation (\%) on test sets in case of a 20\% noise and 10 samples per class.}
	\newcolumntype{L}[1]{>{\hsize=#1\hsize\raggedright\arraybackslash}X}%
	\newcolumntype{C}[1]{>{\hsize=#1\hsize\centering\arraybackslash}X}%
	\label{table:Noise}
	
	\begin{tabularx}{\linewidth}{L{5.8}C{6.4}C{5.4}C{5.4}}
		\hline \hline 
		& Indian Pines & Botswana & Pavia Uni. \\ \hline \\
		EPOCHS=50\\ \\
		Rank-1 FNN   & $61.10 \pm 2.54$ & $ 92.40 \pm 1.59$ & $72.21   \pm 3.46 $\\ \hline
		Rank-2 FNN   & $61.82  \pm 2.72$ & $93.44  \pm 0.77$ & $73.67\pm 3.16$\\ \hline 
		Rank-3 FNN   & $61.92  \pm 2.15$ & $91.75 \pm 2.31$ & $69.87  \pm 3.25$\\ \hline 
		Rank-4 FNN   & $60.80  \pm 2.49$ & $91.98  \pm 0.88$ & $71.07  \pm 2.66$\\ \hline 
		Rank-5 FNN   & $60.79  \pm 2.35$ & $92.63   \pm 0.96$ & $69.29  \pm 5.62$\\ \hline
		CNN  		 & $39.64  \pm 10.5$ & $58.73  \pm 10.2$ & $54.58 \pm 6.79$\\ \hline \hline \\
		EPOCHS=500\\ \\
		Rank-1 FNN   & $60.33  \pm 2.73$ & $94.42   \pm 0.74$ & $72.43 \pm 4.23$\\ \hline
		Rank-2 FNN   & $61.38  \pm 2.89$ & $93.82 \pm 0.47$ & $73.94   \pm 5.33$\\ \hline
		Rank-3 FNN   & $59.67 \pm 2.83$ & $93.62 \pm 1.99$ & $71.88   \pm 2.88$\\ \hline 
		Rank-4 FNN   & $60.70   \pm 3.46$ & $ 93.98 \pm 0.92$ & $71.36  \pm 0.78$\\ \hline 
		Rank-5 FNN   & $61.24   \pm 3.11$ & $92.41  \pm 1.48$ & $70.10  \pm 5.68$\\ \hline
		CNN  		 & $59.78  \pm 2.74$ & $91.12  \pm 1.69$ & $69.99  \pm 1.99$\\ \hline \hline\\
		
	\end{tabularx}
\end{table}

Table \ref{table:Noise} presents the results of classification accuracy on test sets of our proposed model and the CNN in case of a 20\% noise. We depict the results for 50 and 500 epochs. We observe that in all cases, the proposed Rank-$R$ model presents better classification accuracy than the CNN, though times smaller number of parameters are employed (see Table \ref{table:2}). In particular, in the Pavia University dataset, our model is $3.9\%$ better than CNN even for 500 epochs, while for 50 epochs, the improvement reaches $19.1\%$. As for Botswana, the improvement is $3.3\%$ in the case of 500 epochs and $33.9\%$ for 50 epochs. Finally, for the Indian Pines, the improvement is $1.6\%$ for 500 epochs and $22.2\%$ for 50 epochs. These numbers show that our model rapidly converges than the CNN, which constitutes another advantage of our method.

Contrary to CNN, tensor-based models seem to be noise-robust, since their performance has small variations across  different levels of noise.
The significant reduction of the number of trainable parameters, shields them against over-fitting, and their performance has small variations across different levels of noise applied on the data. Besides, our Rank-$R$ model improves classification performance compared to CNN when noise is added to the data. Thus, our model is more robust than the CNN model.

\section{Conclusions-Discussions}
\label{sec:conclusions}
In this work we present  Rank-$R$ FNN,
a tensor-based non-linear classifier
that imposes a CP decomposition on its weight parameters that connect the input layer to the first hidden layer.
Varying the rank of weights decomposition can be seen as a regularization technique that
affects the learning capacity of the model and
shows off robustness to overfitting.
We proved that Rank-$R$ FNN models
are universal approximators and form a learnable class of functions. Experiments on three publicly available  high-order hyperspectral image datasets show that the proposed model has robustness against noise especially when a small number of training samples is selected, smaller standard deviation over different execution runs and it rapidly converges with respect to the training epochs. In particular, the main conclusions are the following:

\begin{itemize}
	\item The number of the parameters of the proposed Rank-$R$ model is times smaller than the number of the respective parameters of conventional CNN networks like the one of \cite{makantasis2015deep}. In particular, the  CNN models need 41, 48 and 57  times more  parameters  compared  to  the  Rank-1  FNN  for  the  Indian  Pines,  the  Pavia  University  and  the Botswana datasets, respectively.
	\item The proposed Rank-$R$ model converges much more rapidly compared to the traditional CNN network. More specifically, our models converges in less than 20 epochs for all the examined datasets, while CNN requires more than 200 epochs for its convergence. 
	\item The standard deviation of the classification accuracy of our model is much smaller than the one achieved by the use of the CNN over different execution runs. Particularly, the average reduction of the standard deviation is on average $3.8$ times over all datasets for the noise-free examples and on average $20\%$ better in case of noisy data inputs.
	\item The proposed Rank-$R$ model is robust against noisy input data retaining both convergence rate efficiency, (less than 20 epochs and classification accuracy. Indeed, in case of adding noise, our model outperforms CNN for all datasets and when convergence of 500 epochs is achieved. Moreover, the improvement of our model to the CNN increases as more noise is added to the input data. 
	\item The classification accuracy of the proposed model is statistically better when a small number of training samples per class are selected (that is 10 samples for 20 epochs), while the improvement decreases for large training epochs of 500; for noise-free data, in two out of the three datasets, the proposed Rank-$R$ model slightly outperforms the CNN model by $3.8\%$ and $1.8\%$ for Botswana and Pavia University datasets, while in the Indian Pines dataset is slightly worse by $4.3\%$. For noisy data, and particularly when a noise of $20\%$ is added, our model outperforms CNN for all the three examined datasets, that is, it is $3.9\%$ better for the  Pavia  University dataset, $3.3\%$ better for the Botswama dataset and $1.6\%$ for the Indian Pines dataset.   
	
\end{itemize} 

\appendices
\section*{Appendix A}
\begin{proof}[Proof of Rank-$R$ FNN Theorem]
	We have to show that for
	any $f$
	there exists $g$ such that $g(\bm A) = f(\bm A)$,
	for any tensor object $\bm A \in \mathbb R^{p_1\times\cdots\times p_D}$.
	Functions $f$ and $g$ can be written as
	\begin{equation}
		f = f_2(f_1) \:\:\: \text{and} \:\:\: g = g_2(g_1), \nonumber
	\end{equation}
	where $f_1$ and $g_1$ represent the output of the hidden layer,
	and
	functions $f_2$ and $g_2$ map the hidden layer outputs to the output layer. Since these two networks have the same number of hidden neurons we can set $g_2 = f_2$.
	In order to complete the proof, it suffices to show that
	there exists $g_1$ such that $g_1(\bm A) = f_1(\bm A)$.
	
	Functions $f_1$ and $g_1$ are vector functions, that is
	\begin{equation}
		f_1 = [f_1^1, f_1^2, \cdots f_1^Q]^T \:\:\: \text{and} \:\:\: g_1 = [g_1^1, g_1^2, \cdots g_1^Q]^T. \nonumber
	\end{equation}
	Without loss of generality assume that
	$f_1^q=\sigma$ for all $q \in \{1, \dots, Q\}$,
	where $\sigma$
	is the sigmoid activation function.
	Under the previous assumption,
	we can set $g_1^q=\sigma$ for all $q$.
	
	From now on,
	with an abuse of notation,
	we remove the superscript $q$ and refer
	to  hidden units $f$ and $g$.
	What remains to show is that,
	for any $\bm w$, there exist decomposition
	$
	\sum_{r=1}^R \bm b_1^{(r)}\circ \cdots \bm b_D^{(r)}
	$
	such that
	\begin{equation}
		\sigma(\langle \bm w, \bm A \rangle) \:\:\:  = \sigma(\langle \sum_{r=1}^R \bm b_1^{(r)}\circ \cdots \bm b_D^{(r)}, \bm A \rangle) \nonumber
	\end{equation}
	But  any $\bm w \in \mathbb{R}^{\prod_1^D p_d}$
	admits a rank $R$
	decomposition
	for some finite $R$,
	i.e.,
	\begin{equation}
		\label{eq:equal}
		\text{ten}(\bm w) = \sum_{r=1}^R \bm b_1^{(r)}\circ \cdots \bm b_D^{(r)},
	\end{equation}
	or equivalently
	\begin{equation}
		\text{vec}(\sum_{r=1}^R \bm b_1^{(r)}\circ \cdots \bm b_D^{(r)}) = \bm w. \nonumber
	\end{equation}
	Equation (\ref{eq:equal}) holds for every tensor $\text{ten}(\bm w) \in \mathbb R^{p_1 \times \cdots \times p_D}$,
	if its rank, let's say $R'$,
	is less than or equal to $R$.
	An upper bound for the rank $R'$ of such  tensor $\text{ten}(\bm w)$ is the following:
	\begin{equation}
		R' \leq \min_i P_i \:\:\:\text{with} \:\:\ P_i = \prod_{d\neq i}p_d. \nonumber
	\end{equation}
	Therefore, for $R \geq \min_i P_i$, there exist $\bm b_d^{(r)}$ for $d=1,\cdots, D$, $r=1,\cdots,R$ such that equation (\ref{eq:equal}) holds. The above arguments hold for every $q=1,\cdots,Q$, which implies that there exists $g_1$ such that $g_1(\bm A) = f_1(\bm A)$ for any $\bm A$.
\end{proof}

\ifCLASSOPTIONcompsoc
  \section*{Acknowledgments}
\else
  \section*{Acknowledgment}
\fi

This paper is supported by the European Union funded project HYPERION "Development of a Decision Support System for Improved Resilience \& Sustainable Reconstruction of historic areas to cope with Climate Change \& Extreme Events based on Novel Sensors and Modelling Tools," with grant agreement 821054, the project Panoptis "Development of a Decision Support System for increasing the Resilience of Transportation Infrastructure based on combined use of terrestrial and airborne sensors and advanced modelling tools' with grant agreement 769129 both funded under the Horizon 2020 research program. This work has also been supported by the European Union's Horizon 2020 research and innovation programme from the TAMED project with Grant Agreement No. 101003397.

\ifCLASSOPTIONcaptionsoff
  \newpage
\fi

\end{document}